\documentclass{article}
\usepackage{amssymb,amsbsy,amsmath,amscd,amsthm,amsfonts}
\usepackage{cite}
\usepackage{enumerate}
\usepackage{dsfont}
\usepackage{graphicx}
\usepackage{longtable}
\usepackage{tablefootnote}
\usepackage{standalone}
\usepackage{fancyhdr}
\usepackage{float}
\usepackage{caption}
\usepackage{subcaption}
\usepackage{url}
\RequirePackage[colorlinks = true,
            linkcolor = black,
            urlcolor  = black,
            citecolor = black,
            anchorcolor = black]{hyperref}

\newcommand{\Var}{\mathsf{Var}}

\newcommand{\indc}{\mathbf{1}}

\newcommand{\calF}{{\mathcal{F}}}

\newtheorem{theorem}{Theorem}

\newtheorem{lemma}{Lemma}
\newtheorem{remark}{Remark}

\usepackage{prettyref}
\newrefformat{eq}{(\ref{#1})}
\newrefformat{thm}{Theorem~\ref{#1}}
\newrefformat{sec}{Section~\ref{#1}}
\newrefformat{algo}{Algorithm~\ref{#1}}
\newrefformat{fig}{Fig.~\ref{#1}}
\newrefformat{tab}{Table~\ref{#1}}
\newrefformat{rmk}{Remark~\ref{#1}}
\newrefformat{def}{Definition~\ref{#1}}
\newrefformat{cor}{Corollary~\ref{#1}}
\newrefformat{lmm}{Lemma~\ref{#1}}
\newrefformat{prop}{Proposition~\ref{#1}}
\newrefformat{app}{Appendix~\ref{#1}}
\newrefformat{ex}{Example~\ref{#1}}
\newrefformat{ass}{Assumption~\ref{#1}}

\begin{document}

\title{Approximation by Combinations of ReLU and Squared ReLU Ridge Functions with $ \ell^1 $ and $ \ell^0 $ Controls}

\author{Jason~M.~Klusowski \and
    Andrew~R.~Barron
\thanks{Jason M. Klusowski and Andrew R. Barron are with the Department of Statistics and Data Science, Yale University, New Haven, CT, USA, 06511, e-mail: \{jason.klusowski, andrew.barron\}@yale.edu.}
\thanks{}}

\maketitle

\begin{abstract}
We establish $ L^{\infty} $ and $ L^2 $ error bounds for functions of many variables that are approximated by linear combinations of ReLU (rectified linear unit) and squared ReLU ridge functions with $ \ell^1 $ and $ \ell^0 $ controls on their inner and outer parameters. With the squared ReLU ridge function, we show that the $ L^2 $ approximation error is inversely proportional to the inner layer $ \ell^0 $ sparsity and it need only be sublinear in the outer layer $ \ell^0 $ sparsity. Our constructions are obtained using a variant of the Jones-Barron probabilistic method, which can be interpreted as either stratified sampling with proportionate allocation or two-stage cluster sampling. We also provide companion error lower bounds that reveal near optimality of our constructions. Despite the sparsity assumptions, we showcase the richness and flexibility of these ridge combinations by defining a large family of functions, in terms of certain spectral conditions, that are particularly well approximated by them. 
\end{abstract}

\section{Introduction}

Functions of many variables are approximated using linear combinations of ridge functions with one layer of nonlinearities, viz.,
\begin{equation} \label{eq:linear}
f_m(x) = \sum_{k=1}^m  b_k \phi( a_k \cdot x - t_k),
\end{equation}
where $ b_k \in \mathbb{R} $ are the outer layer parameters and $ a_k \in \mathbb{R}^d $ are the vectors of inner parameters for the single-hidden layer of functions $ \phi( a_k \cdot x - t_k) $. The activation function $ \phi $ is allowed to be quite general. For example, it can be bounded and Lipschitz, polynomials with certain controls on their degrees, or bounded with jump discontinuities. When the ridge activation function is a sigmoid, \prettyref{eq:linear} is single-hidden layer artificial neural network.

One goal in a statistical setting is to estimate a regression function, i.e., conditional mean response, $ f(x) = \mathbb{E}[Y \mid X = x] $ with domain $ D \triangleq [-1,1]^d $ from noisy observations $ \{(X_i,Y_i)\}_{i=1}^n $, where $ Y = f(X) + \varepsilon $. In classical literature \cite{Barron1994}, $ L^2(P) $ mean squared prediction error of order $ (d/n)^{1/2} $, achieved by $ \ell^1 $ penalized least squares estimators\footnote{That is, the fit minimizes $ (1/n) \sum_{i=1}^n (f_m(X_i) - Y_i)^2 + \lambda \sum_{k=1}^m |b_k| $ for some appropriately chosen $ \lambda > 0 $.} over the class of models \prettyref{eq:linear}, are obtained by optimizing the tradeoff between \emph{approximation error} and \emph{descriptive complexity relative to sample size}. Bounds on the approximation error are obtained by first showing how models of the form \prettyref{eq:linear} with $ \phi(z) = \indc\{z > 0\} $ can be used to approximate $ f $ satisfying $ \int_{\mathbb{R}^d}\|\omega\|_1|\calF(f)(\omega)|d\omega < +\infty $, provided $ f $ admits a Fourier representation $ f(x) = \int_{\mathbb{R}^d}e^{ix\cdot\omega}\calF(f)(\omega)d\omega $ on $ [-1,1]^d $. Because it is often difficult to work with discontinuous $ \phi $ (i.e., vanishing or exploding gradient issues), these step functions are replaced with smooth $ \phi $ such that $ \phi(\tau z) \wedge 1 \rightarrow \indc\{z > 0\} $ as $ \tau \rightarrow +\infty $. Thus, this setup allows one to work with approximants of the form \prettyref{eq:linear} with smooth $ \phi $, but at the expense of \emph{unbounded} $ \ell^1 $ norm $ \| a_k\|_1 $.

Like high-dimensional linear regression \cite{Wainwright2011}, many applications of statistical inference and estimation require a setting where $ d \gg n $. In contrast to the aforementioned mean square prediction error of $ (d/n)^{1/2} $, it has been shown \cite{Klusowski2018} how models of the form \prettyref{eq:linear} with Lipschitz\footnote{Henceforth, when we say a function is Lipschitz, we assume it has bounded Lipschitz parameter.} $ \phi $ (reps. Lipschitz derivative $ \phi' $) and \emph{bounded} inner parameters $ \| a_k\|_0 $ and $ \|a_k\|_1 $ can be used to give desirable $ L^2(D) $ mean squared prediction error of order $ ((\log d)/n)^{1/3} $ (resp. $ ((\log d)/n)^{2/5} $), also achieved by penalized estimators.\footnote{With additional $ \ell^0 $ inner sparsity, we might also consider an estimator that minimizes $ (1/n) \sum_{i=1}^n (f_m(X_i) - Y_i)^2 + \lambda_0\psi\left(\sum_{k=1}^m |b_k|\|a_k\|_0\right) $ for some convex function $ \psi $ and appropriately chosen $ \lambda_0 > 0 $.} In fact, \cite{Klusowski2017} shows that these rates are nearly optimal. A few natural questions arise from restricting the $ \ell^0 $ and $ \ell^1 $ norms of the inner parameters in the model:
 
\begin{itemize}
\item To what degree do the sparsity assumptions limit the flexibility of the model \prettyref{eq:linear}?
\item What condition can be imposed on $ f $ so that it can be approximated by $ f_m $ with Lipschitz $ \phi $ (or Lipschitz derivative $ \phi' $) and bounded $ \| a_k\|_0 $ and / or  $ \| a_k\|_1 $?
\item How well can $ f $ be approximated by $ f_m $, given these sparsity constraints?
\end{itemize}

According to classic approximation results \cite{Barron1992, Barron1993}, if the domain of $ f $ is contained in $ [-1,1]^d $ and $ f $ admits a Fourier representation $ f(x) = \int_{\mathbb{R}^d}e^{ix\cdot\omega}\calF(f)(\omega)d\omega $, then the spectral condition $ v_{f,1} < +\infty $, where $ v_{f,s} \triangleq \int_{\mathbb{R}^d}\|\omega\|^{s}_1|\calF(f)(\omega)|d\omega $, is enough to ensure that $ f - f(0) $ can be approximated in $ L^{\infty}(D) $ by equally weighted, i.e, $ | b_1| = \cdots = | b_m| $, linear combinations of functions of the form \prettyref{eq:linear} with $ \phi(z) = \indc\{z > 0\} $. Typical $ L^{\infty} $ error rates $ \| f - f_m\|_{\infty} $ of an $ m $-term approximation \prettyref{eq:linear} are at most $ cv_{f,1}\sqrt{d}\ m^{-1/2} $, where $ c $ is a universal constant \cite{Barron1992, Cheang2000, Yukich1995}. A rate of $ c(p) v_{f,1}m^{-1/2-1/(pd)} $ was given in \cite[Theorem 3]{Makovoz1996} for $ L^p(D) $ for nonnegative even integer $ p $. Again, all these bounds are valid when the step activation function is replaced by a smooth approximant $ \phi $ (in particular, \emph{any} sigmoid satisfying $ \lim_{z\rightarrow \pm \infty}\phi(z) = \pm 1 $), but at the expense of unbounded $ \| a_k\|_1 $.

Towards giving partial answers to the questions we posed, in \prettyref{sec:uniform}, we show how functions of the form \prettyref{eq:linear} with ReLU (also known as a ramp or first order spline) $ \phi(z) = (z)_{+} = 0 \vee z $ (which is Lipchitz)\footnote{It is perhaps more conventional to write $ (z)^{+} $ for $ 0 \vee z $, however, to avoid clutter in the exponent, we use the current notation.} or squared ReLU $ \phi(z) = (z)^2_{+} $ (which has Lipschitz derivative) activation function can be used  to give desirable $ L^{\infty}(D) $ approximation error bounds, even when $ \| a_k\|_1 = 1 $, $ 0 \leq t_k \leq 1 $, and $ | b_1| = \dots = | b_m| $. Because of the widespread popularity of the ReLU activation function and its variants, these simpler forms may also be of independent interest for computational and algorithmic reasons as in \cite{Bartlett1998, Montanari2017, Janzamin2015, Zhang2015, Globerson2017}, to name a few.

Unlike the case with step activation functions, our analysis makes no use of the combinatorial properties of half-spaces as in Vapnik-Chervonenkis theory \cite{Vapnik1971, Haussler1995}. The $ L^2(D) $ case for ReLU ridge functions (also known as hinging hyperplanes) with $ \ell^1 $-bounded inner parameters was considered in \cite[Theorem 3]{Breiman1993} and our $ L^{\infty}(D) $ bounds improve upon that line of work and, in addition, increase the exponent from $ 1/2 $ to $ 1/2 + O(1/d) $. Our proof techniques are substantively different than \cite{Breiman1993} and, importantly, are more amenable to empirical process theory, which is the key to showing our error bounds.

These tighter rates of approximation, with ReLU and squared ReLU activation functions, are possible under two different conditions -- finite $ v_{f,2} $ or $ v_{f,3} $, respectively. 
The main idea we use originates from \cite{Makovoz1996} and \cite{Makovoz1998} and can be seen as stratified sampling with proportionate allocation. This technique is widely applied in survey sampling as a means of variance reduction \cite{Neyman1934}.

At the end of \prettyref{sec:uniform}, we will also discuss the degree to which these bounds can be improved by providing companion lower bounds on the minimax rates of approximation. 

\prettyref{sec:squared} will focus on how accurate estimation can be achieved even when $ \| a_k\|_0 $ is also bounded. In particular, we show how an $ m $-term linear combination \prettyref{eq:linear} with $ \| a_k\|_0 \leq \sqrt{m} $ and $ \| a_k\|_1 = 1 $ can approximate $ f $ satisfying $ v_{f, 3} < +\infty $ in $ L^2(D) $ with error at most $ \sqrt{2}v_{f, 3}m^{-1/2} $. In other words, the $ L^2(D) $ approximation error is inversely proportional to the inner layer sparsity and it need only be sublinear in the outer layer sparsity. The constructions that achieve these error bounds are obtained using a variant of the Jones-Barron probabilistic method, which can be interpreted as two-stage cluster sampling.


Throughout this paper, we will state explicitly how our bounds depend on $ d $ so that the reader can fully appreciate the complexity of approximation. If $ a $ is a vector in Euclidean space, we use the notation $ a(k) $ to denote its $ k $-th component.

\section{$ L^{\infty} $ approximation with bounded $ \ell^1 $ norm} \label{sec:uniform}
\subsection{Positive results}
In this section, we provide the statements and proofs of the existence results for $ f_m $ with bounded $ \ell^1 $ norm of inner parameters. We would like to point out that the results of \prettyref{thm:general} hold when all occurrences of the ReLU or squared ReLU activation functions are replaced by general $ \phi $ which is Lipschitz or has Lipschitz derivative $ \phi^{\prime} $, respectively.
\begin{theorem} \label{thm:general}
Suppose $ f $ admits an integral representation
\begin{equation} \label{eq:integral}
f(x) = v\int_{[0,1]\times \{ a:\| a\|_1 = 1\}}\eta(t, a) \ ( a\cdot x - t)^{s-1}_{+}dP(t, a),
\end{equation}
for $ x $ in $ D = [-1,1]^d $ and $ s \in \{2, 3\} $, where $ P $ is a probability measure on $ [0,1]\times \{ a \in \mathbb{R}^d:\| a\|_1 = 1\}  $ and $ \eta(t, a) $ is either $ -1 $ or $ +1 $. There exists a linear combination of ridge functions of the form
\begin{equation}
 f_m(x) =  \frac{v}{m}\sum_{k=1}^m b_k(a_k \cdot x-t_k)^{s-1}_{+},
\end{equation}
with $  b_k \in [-1,1] $, $ \| a_k\|_1 = 1 $, $ 0 \leq t_k \leq 1 $ such that
\begin{equation*}
\sup_{x\in D}|f(x)-f_m(x)| \leq c\sqrt{d + \log m}\ m^{-1/2-1/d}, \quad s = 2,
\end{equation*}
and
\begin{equation*}
\sup_{x\in D}|f(x)-f_m(x)| \leq c\sqrt{d}\ m^{-1/2-1/d}, \quad s = 3,
\end{equation*}
for some universal constant $ c > 0 $. Furthermore, if the $  b_k $ are restricted to $ \{-1,1\} $, the upper bound is of order
\begin{equation*}
\sqrt{d + \log m}\ m^{-1/2-1/(d+2)}, \quad s = 2
\end{equation*}
and
\begin{equation*}
\sqrt{d}\ m^{-1/2-1/(d+2)}, \quad s = 3.
\end{equation*}
\end{theorem}

\begin{theorem} \label{thm:ramp}
Let $ D = [-1,1]^d $. Suppose $ f $ admits a Fourier representation $ f(x) = \int_{\mathbb{R}^d}e^{ix\cdot\omega}\calF(f)(\omega)d\omega $ and
\begin{equation*} v_{f,2} = \int_{\mathbb{R}^d}\|\omega\|^2_1|\calF(f)(\omega)|d\omega < +\infty. 
\end{equation*}
There exists a linear combination of ReLU ridge functions of the form
\begin{equation} \label{thm:ramp_equal}
 f_m(x) =  b_0 + a_0\cdot x + \frac{v}{m}\sum_{k=1}^m b_k(a_k \cdot x-t_k)_{+}
\end{equation}
 with $  b_k \in [-1,1] $, $ \| a_k\|_1 = 1 $, $ 0 \leq t_k \leq 1 $, $  b_0 = f(0) $, $  a_0 = \nabla f(0) $, and $ v \leq 2v_{f,2} $ such that
\begin{equation*}
\sup_{x\in D}|f(x)-f_m(x)| \leq cv_{f,2}\sqrt{d + \log m}\ m^{-1/2-1/d},
\end{equation*}
for some universal constant $ c > 0 $. Furthermore, if the $  b_k $ are restricted to $ \{-1,1\} $, the upper bound is of order
\begin{equation*}
v_{f,2}\sqrt{d + \log m}\ m^{-1/2-1/(d+2)}.
\end{equation*}
\end{theorem}

\begin{theorem} \label{thm:smooth_spline}
Under the setup of \prettyref{thm:ramp}, suppose 
\begin{equation*} v_{f,3} = \int_{\mathbb{R}^d}\|\omega\|^3_1|\calF(f)(\omega)|d\omega < +\infty. 
\end{equation*}
There exists a linear combination of squared ReLU ridge functions of the form
\begin{equation} \label{thm:smooth_spline_equal}
 f_m(x) =  b_0 + a_0\cdot x + x^TA_0x + \frac{v}{2m}\sum_{k=1}^m b_k(a_k \cdot x-t_k)^2_{+}
\end{equation}
 with $  b_k \in [-1,1] $, $ \| a_k\|_1 = 1 $, $ 0 \leq t_k \leq 1 $, $  b_0 = f(0) $, $  a_0 = \nabla f(0) $, $ A_0 = \nabla\nabla^Tf(0) $, and $ v \leq 2v_{f,3} $ such that
\begin{equation*}
\sup_{x\in D}|f(x)-f_m(x)| \leq cv_{f,3}\sqrt{d}\ m^{-1/2-1/d},
\end{equation*}
for some universal constant $ c > 0 $. Furthermore, if the $  b_k $ are restricted to $ \{-1,1\} $, the upper bound is of order
\begin{equation*}
v_{f,3}\sqrt{d}\ m^{-1/2-1/(d+2)}.
\end{equation*}
\end{theorem}

The key observation for proving \prettyref{thm:ramp} and \prettyref{thm:smooth_spline} is that $ f $ modulo linear or quadratic terms with finite $ v_{f, s} $ can be written in the integral form \prettyref{eq:integral}. Unlike in \cite[Theorem 3]{Breiman1993} where an interpolation argument is used, our technique of writing $ f $ as the mean of a random variable allows for more straightforward use of empirical process theory to bound the expected sup-error of the empirical average of $ m $ independent draws from its population mean. Our argument is also more flexible than \cite{Breiman1993} and can be readily adapted to the case of squared ReLU activation function. We should also point out that our $ L^{\infty}(D) $ error bounds immediately imply $ L^p(D) $ error bounds for all $ p \geq 1 $. In fact, using nearly exactly the same techniques, it can be shown that the results in \prettyref{thm:general}, \prettyref{thm:ramp}, and \prettyref{thm:smooth_spline} hold verbatim in $ L^2(D) $, sans the $ \sqrt{d+\log m} $ or $ \sqrt{d} $ factors, corresponding to the ReLu or squared ReLU cases, respectively.

\begin{remark}
In \cite{Makovoz1998}, it was shown that the standard order $ m^{-1/2} $ $ L^{\infty}(D) $ error bound alluded to earlier could be improved to be of order $ \sqrt{\log m}\ m^{-1/2-1/(2d)} $ under an alternate condition of finite $ v^{\star}_{f,1} \triangleq \sup_{u\in\mathbb{S}^{d-1}}\int_{0}^{\infty}r^d|\calF(f)(ru)|dr $, but with the requirement that $ \| a_k\|_1 $ be unbounded. In general, our assumptions are neither stronger nor weaker than this since the function $ f $ with Fourier transform $ \calF(f)(\omega) = e^{-\|\omega-\omega_0\|}/\|\omega-\omega_0\| $ for $ \omega_0 \neq 0 $ and $ d \geq 2 $ has infinite $ v^{\star}_{f,1} $ but finite $ v_{f,s} $ for $ s \geq 0 $, while the function $ f $ with Fourier transform $ \calF(f)(\omega) = 1/(1+\|\omega\|)^{d+2} $ has finite $ v^{\star}_{f,1} $ but infinite $ v_{f,s} $ for $ s \geq 2 $.
\end{remark}

\begin{proof}[Proof of \prettyref{thm:general}]
\noindent \emph{\bf{Case I: $ s = 2 $.}}  Let $ \mathcal{B}_1,\dots,\mathcal{B}_M $ be a partition of the space $ \Omega = \{ (\eta, t, a)^{\prime}: \eta \in\{-1,+1\},\  0 \leq t \leq 1,\  \| a\|_1 = 1 \} $ such that
\begin{equation} \label{eq:eq4}
\inf_{(\widetilde{\eta}, \widetilde{t}, \widetilde{a})^{\prime} \in \mathcal{B}_k, \  k = 1, \dots, M}\sup_{(\eta, t, a)^{\prime} \in \Omega} \| h(\widetilde{\eta}, \widetilde{t}, \widetilde{a}) - h(\eta, t, a) \|_{\infty} < \epsilon,
\end{equation}
where $ h(\eta, t, a)(x) = h(x) = \eta ( a \cdot x - t)^{s-1}_{+} $.
It is not hard to show that $ M \asymp \epsilon^{-d} $.
For $ k = 1,\dots, M $ define 
\begin{equation*}
dP_{k}(t, a) = dP(t, a)\indc\{ (\eta(t,  a), t,  a)^{\prime} \in \mathcal{B}_k \}/L_k,
\end{equation*}
where $ L_k $ is chosen to make $ P_k $ a probability measure. A very important property we will use is that $ \Var_{P_k}[h] \leq \epsilon $, which follows from \prettyref{eq:eq4}.
Let $ m $ be a positive integer and define a sequence of $ M $ independent random variables $ \{m_{k}\}_{1 \leq k\leq M} $ as follows: let $ m_{k} $ equal $ \lfloor mL_{k} \rfloor $ and $ \lceil mL_k \rceil $ with probabilities chosen to make its mean equal to $ mL_{k} $. Given, $ \underline{m} = \{m_{k}\}_{1 \leq k\leq M} $, take a random sample $  \underline{a} = \{(t_{j,k}, a_{j,k})^{\prime}\}_{1 \leq j
\leq n_k,\  1 \leq k \leq M} $ of size $ n_{k} = m_k + \indc\{m_k=0\} $ from $ P_k $. Thus, we split the population $ \Omega $ into $ M $ ``strata'' $ \mathcal{B}_1,\dots,\mathcal{B}_M $ and allocate the number of within-stratum samples to be proportional to the ``size'' of the stratum $ m_1,\dots,m_M $ (i.e., proportionate allocation). The within-stratum variability of $ h $ (i.e., $ \Var_{P_k}[h] $) is now smaller than the population level variability (i.e., $ \Var_P[h] $) by a factor of $ \epsilon $ as evidenced by \prettyref{eq:eq4}.

Note that the $ n_{k} $ sum to be at most $ m + M $ because
\begin{align}
\sum_{k=1}^M n_k
& = \sum_{k=1}^M m_k\indc\{m_k > 0\} + \sum_{k=1}^M\indc\{m_k=0\} \nonumber \\
& \leq \sum_{k=1}^M(mL_k+1)\indc\{m_k > 0\} + \sum_{k=1}^M\indc\{m_k=0\} \nonumber \\
& = m\sum_{k=1}^ML_k\indc\{m_k > 0\} + M \nonumber \\
& \leq m + M, \label{eq:eq0}
\end{align}
where the last inequality follows from $ \sum_{k=1}^M L_k \leq 1 $.
For $ j = 1,\dots,m_k $, let $ h_{j,k} = h(\eta(t_{j,k},  a_{j,k}),t_{j,k}, a_{j,k}) $ and $ f_k = \frac{vm_k}{mn_k}\sum_{j=1}^{n_k}h_{j,k} $. Also, let  $ \overline{f}_m= \sum_{k=1}^Mf_k $. A simple calculation shows that the mean of $ \overline{f}_m$ is $ f $. Write $ \sum_{k=1}^M(f_k(x)-\mathbb{E}f_k(x)) = \frac{v}{m}\left(\sum_{k=1}^M(m_k-L_km)\mathbb{E}_{P_k}h(x)\right) + \frac{v}{m}\left(\sum_{k=1}^M\sum_{j=1}^{n_k}\frac{m_k}{n_k}(h_{j,k}(x)-\mathbb{E}_{P_k}h(x))\right) $. By the triangle inequality, we upper bound 
\begin{equation*}
\mathbb{E}\sup_{x\in D}|\overline{f}_m(x)-f(x)| =  \mathbb{E}\sup_{x\in D}|\sum_{k=1}^M(f_k(x)-\mathbb{E}f_k(x))|
\end{equation*}
by
\begin{align}
& \frac{v}{m}\mathbb{E}_{\underline{m}}\sup_{x\in D}|\sum_{k=1}^M(m_k-L_km)\mathbb{E}_{P_k}h(x)| + \nonumber \\ & \frac{v}{m}\mathbb{E}_{\underline{m}}\mathbb{E}_{ \underline{a}\mid\underline{m}}\sup_{x\in D}|\sum_{k=1}^M\sum_{j=1}^{n_k}\frac{m_k}{n_k}(h_{j,k}(x)-\mathbb{E}_{P_k}h(x))|. \label{eq:eq3}
\end{align}
Now
\begin{align}
\mathbb{E}_{ \underline{a}\mid\underline{m}}\sup_{x\in D}|\sum_{k=1}^M\sum_{j=1}^{n_k}\frac{m_k}{n_k}(h_{j,k}(x)-\mathbb{E}_{P_k}h(x))| \leq & \nonumber \\
2\mathbb{E}_{ \underline{a}\mid\underline{m}}\sup_{x\in D}|\sum_{k=1}^{M}\sum_{j=1}^{n_k}\sigma_{j,k}\frac{m_k}{n_k}[h_{j,k}(x)-\mu_{j,k}(x)]|, \label{eq:eq1}
\end{align}
where $ \{ \sigma_{j,k} \} $ is a sequence of independent identically distributed Rademacher variables and $ \{ x\mapsto\mu_{j,k}(x) \} $ is any sequence of functions defined on $ D $ [see for example Lemma 2.3.6 in \cite{Wellner1996}]. For notational brevity, we define $ \widetilde{h}_{j,k}(x) = \frac{m_k}{n_k}[h_{j,k}(x)-\mu_{j,k}(x)] $.
By Dudley's entropy integral method [see Corollary 13.2 in \cite{Boucheron2013}], the quantity in \prettyref{eq:eq1} can be bounded by
\begin{equation} \label{eq:Dudley}
24\int_{0}^{\delta/2}\sqrt{N(u, D)}du,
\end{equation}
where $ N(u, D) $ is the $ u $-metric entropy of $ D $ with respect to the norm $ \kappa(x,x^{\prime}) $ (i.e., the logarithm of the smallest $ u $-net that covers $ D $ with respect to $ \kappa $) defined by
\begin{align}
\kappa^2(x,x^{\prime})
& \triangleq \sum_{k=1}^{M}\sum_{j=1}^{n_k}(\widetilde{h}_{j,k}(x) - \widetilde{h}_{j,k}(x^{\prime}))^2 \nonumber \\
& \leq (m+M)\|x-x^{\prime}\|^2_{\infty}, \label{eq:eq2}
\end{align}
and $ \delta^2 = \sup_{x\in D}\sum_{k=1}^{M}\sum_{j=1}^{n_k}|\widetilde{h}_{j,k}(x)|^2 $. If we set $ \mu_{j,k} $ to equal $ \frac{m_k}{n_k}h(\eta(t_k,  a_k),t_k, a_k) $, where $ (\eta_k,t_k, a_k)^{\prime} $ is any fixed point in $ \mathcal{B}_k $, it follows from \prettyref{eq:eq4} and \prettyref{eq:eq0} that $ \delta \leq \sqrt{m+M}\epsilon $ and from \prettyref{eq:eq2} that $ N(u, D) \leq d\log(3\sqrt{m+M}/u) $. By evaluating the integral in \prettyref{eq:Dudley}, we can bound the second term in \prettyref{eq:eq3} by
\begin{equation} \label{eq:second}
24v\sqrt{d}\ m^{-1/2}\epsilon\sqrt{-\log \epsilon + 1}\sqrt{1+M/m}.
\end{equation}
For the first expectation in \prettyref{eq:eq3}, we follow a similar approach. As before,
\begin{align}
& \mathbb{E}_{\underline{m}}\sup_{x\in D}|\sum_{k=1}^M(m_k-L_km)\mathbb{E}_{P_k}h(x)| \nonumber \\ &
\leq 2\mathbb{E}_{\underline{m}}\sup_{x\in D}|\sum_{k=1}^M\sigma_k(m_k-L_km)\mathbb{E}_{P_k}h(x)|, \label{eq:eq6}
\end{align}
where $ \{ \sigma_{k} \} $ is a sequence of independent identically distributed Rademacher variables. For notational brevity, we write $ \widetilde{h}_k(x) = (m_k-L_km)\mathbb{E}_{P_k}h(x) $. We can also bound \prettyref{eq:eq6} by \prettyref{eq:Dudley}, except this time $ N(u, D) $ is the $ u $-metric entropy of $ D $ with respect to the norm $ \rho(x, x^{\prime}) $ defined by
\begin{align}
\rho^2(x, x^{\prime})
& \triangleq \sum_{k=1}^{M}(\widetilde{h}_{k}(x) - \widetilde{h}_{k}(x^{\prime}))^2 \nonumber \\
& \leq M\|x-x^{\prime}\|^2_{\infty}, \label{eq:eq5}
\end{align}
where the last line follows from $ |m_k-L_km| \leq 1 $ and $ |\mathbb{E}_{P_k}h(x)-\mathbb{E}_{P_k}h(x^{\prime})| \leq \|x-x^{\prime}\|_{\infty} $. The quantity $ \delta $ is also less than $ \sqrt{M} $, since $ \sup_{x\in D}|\widetilde{h}_k(x)| \leq 1 $ and moreover $ N(u, D) \leq d\log(3\sqrt{M}/u) $. Evaluating the integral in \prettyref{eq:Dudley} with these specifications yields a bound on the first term in \prettyref{eq:eq3} of
\begin{equation} \label{eq:first}
\frac{48v\sqrt{d}\sqrt{M}}{m}.
\end{equation}
Adding \prettyref{eq:first} and \prettyref{eq:second} together yields a bound on $ \mathbb{E}\sup_{x\in D}|\overline{f}_m(x)-f(x)| $ of
\begin{equation} \label{eq:eq7}
48v\sqrt{d}m^{-1/2}(\sqrt{M/m}+\epsilon\sqrt{1+M/m}\sqrt{-\log \epsilon + 1}).
\end{equation}
Choose
\begin{equation} \label{eq:eq8}
M = m\frac{\epsilon^2(-\log \epsilon+1)}{1-\epsilon^2(-\log \epsilon+1)}.
\end{equation}
Consequently, $ \mathbb{E}\sup_{x\in D}|\overline{f}_m(x)-f(x)| $ is at most
\begin{equation}
96v\sqrt{d}m^{-1/2}\frac{\epsilon\sqrt{-\log\epsilon+1}}{\sqrt{1-\epsilon^2(-\log \epsilon+1)}}.
\end{equation}

We stated earlier that $ M \asymp \epsilon^{-d} $. Thus \prettyref{eq:eq8} determines $ \epsilon $ to be at most of order $ m^{-1/(d+2)} $. Since the inequality \prettyref{eq:eq8} holds on average, there is a realization of $ \overline{f}_m$ for which $ \sup_{x\in D}|\overline{f}_m(x)-f(x)| $ has the same bound. Note that $ \overline{f}_m$ has the desired equally weighted form. 

For the second conclusion, we set $ m_k = mL_k $ and $ n_k = \lceil m_k \rceil $. In this case, the first term in \prettyref{eq:eq3} is zero and hence $ \mathbb{E}\sup_{x\in D}|\overline{f}_m(x)-f(x)| $ is not greater than \prettyref{eq:second}. The conclusion follows with $ M = m $ and $ \epsilon $ of order $ m^{-1/d} $. \\

\noindent \emph{\bf{Case II: $ s = 3 $.}} The metric $ \kappa(x,x^{\prime}) $ is in fact bounded by a constant multiple of $ \sqrt{m+M}\epsilon\|x-x^{\prime}\|_{\infty} $. To see this, we note that the function $ \widetilde{h}_{j,k}(x) $ has the form
\begin{equation*}
\pm\frac{m_k}{n_k}[( a\cdot x-t)^2_{+} - ( a_k \cdot x-t_k)^2_{+}],
\end{equation*}
with $ \| a- a_k\|_1+|t-t_k| < \epsilon $. Thus, the gradient of $ \widetilde{h}_{j,k}(x) $ with respect to $ x $ has the form
\begin{equation*}
\nabla \widetilde{h}_{j,k}(x) = \pm\frac{2m_k}{n_k}[( a( a\cdot x-t)_{+} -  a_k( a_k \cdot x-t_k)_{+}].
\end{equation*}
Adding and subtracting $ \frac{2m_k}{n_k} a( a_k \cdot x-t_k)_{+} $ to the above expression yields the bound of order $ \epsilon $ for $ \sup_{x\in D}\|\nabla \widetilde{h}_{j,k}(x)\|_1 $. Taylor's theorem yields the desired bound on $ \kappa(x,x^{\prime}) $. Again using Dudley's entropy integral, we can bound $ \mathbb{E}\sup_{x\in D}|\overline{f}_m(x)-f(x)| $ by a universal constant multiple of either $ v\sqrt{d}m^{-1/2}(\sqrt{M/m}+\epsilon\sqrt{1+M/m}) $ or $ v\sqrt{d}m^{-1/2}\epsilon\sqrt{1+M/m} $ corresponding to the equally weighted or non-equally weighted cases, respectively. The corresponding results follow with $ M = m\epsilon^2/(1-\epsilon^2) $ and $ \epsilon $ of order $ m^{-1/(d+2)} $ or $ M = m $ and $ \epsilon $ of order $ m^{-1/d} $. Note that here the additional smoothness afforded by the stronger assumption $ v_{f,3} < +\infty $ allows one to remove the $ \sqrt{-\log\epsilon+1} $ factor that appeared in the final bound in the proof of \prettyref{thm:ramp}. This rate is the same as what was achieved in \prettyref{thm:ramp}, without a $ \sqrt{(\log m)/d+1} $ factor.
\end{proof}

\begin{proof}[Proof of \prettyref{thm:ramp}]
If $ |z| \leq c $, we note the identity
\begin{equation}
-\int_{0}^{c}[(z-u)_{+}e^{iu}+(-z-u)_{+}e^{-iu}]du = e^{iz}-iz-1.
\end{equation}
If $ c = \|\omega\|_1 $, $ z = \omega\cdot x $, $  a =  a(\omega) = \omega/\|\omega\|_1 $, and $ u = \|\omega\|_1t $, $ 0 \leq t \leq 1 $, we find that
\begin{align*}
-\|\omega\|^2_1\int_{0}^{1}[( a\cdot x-t)_{+}e^{i\|\omega\|_1t}+(- a\cdot x-t)_{+}e^{-i\|\omega\|_1t}]dt = & \\ e^{i\omega\cdot x}-i\omega\cdot x - 1.
\end{align*}
Multiplying the above by $ \calF(f)(\omega) = e^{i b(\omega)}|\calF(f)(\omega)| $, integrating over $ \mathbb{R}^d $, and applying Fubini's theorem yields
\begin{equation*}
f(x) -x\cdot \nabla f(0) - f(0) = \int_{\mathbb{R}^d}\int_{0}^{1}g(t,\omega)dtd\omega,
\end{equation*} where 
\begin{align*} g(t,\omega) & = -[( a\cdot x-t)_{+}\cos(\|\omega\|_1t+b(\omega)) + \\ &
\qquad (- a\cdot x-t)_{+}\cos(\|\omega\|_1t-b(\omega))]\|\omega\|^2_1|\calF(f)(\omega)|.
\end{align*}
Consider the probability measure on $ \{-1,1\}\times[0,1]\times\mathbb{R}^d $ defined by
\begin{equation}
dP(z,t,\omega) = \frac{1}{v}|\cos(z\|\omega\|_1t+b(\omega))|\|\omega\|^2_1|\calF(f)(\omega)|dt d\omega,
\end{equation}
where
\begin{align*}
& v = \int_{\mathbb{R}^d}\int_{0}^1[|\cos(\|\omega\|_1t+b(\omega))|+ \\ & \qquad |\cos(\|\omega\|_1t-b(\omega))|]\|\omega\|^2_1|\calF(f)(\omega)|dtd\omega \leq 2v_{f,2}.
\end{align*}
Define a function $ h(z,t, a)(x) $ that equals
\begin{equation*}
(z a\cdot x-t)_{+} \ \eta(z, t,\omega),
\end{equation*}
where $ \eta(z, t,\omega) = -\text{sgn}\cos(\|\omega\|_1zt+b(\omega)) $. Note that $ h(z, t, a)(x) $ has the form $ \pm(\pm a\cdot x-t)_{+} $. Thus, we see that
\begin{align} \label{eq:ramp_rep}
& f(x) -x\cdot \nabla f(0) - f(0) = \nonumber \\ & v\int_{\{-1,1\}\times[0,1]\times\mathbb{R}^d}h(z, t, a)(x)dP(z, t, \omega).
\end{align}
The result follows from an application of \prettyref{thm:general}.
\end{proof}

\begin{proof}[Proof of \prettyref{thm:smooth_spline}]
For the result in \prettyref{thm:smooth_spline}, we will use exactly the same techniques. The function $ f(x) - x^T\nabla\nabla^Tf(0)x/2 - x\cdot\nabla f(0) - f(0) $ can be written as the real part of
\begin{equation} \label{eq:rep}
\int_{\mathbb{R}^d}(e^{i\omega\cdot x}+(\omega\cdot x)^2/2-i\omega\cdot x-1)\calF(f)(\omega)d\omega.
\end{equation}
As before, the integrand in \prettyref{eq:rep} admits an integral representation given by
\begin{equation*}
(i/2)\|\omega\|^3_1\int_{0}^{1}[(- a\cdot x-t)^2_{+}e^{-i\|\omega\|_1t}-( a\cdot x-t)^2_{+}e^{i\|\omega\|_1t}]dt,
\end{equation*}
which can be used to show that $ f(x) - x^T\nabla\nabla^Tf(0)x/2 - x\cdot\nabla f(0) - f(0) $ equals
\begin{equation} \label{eq:smooth_spline_rep}
\frac{v}{2}\int_{\{-1,1\}\times[0,1]\times\mathbb{R}^d}h(z, t, a)(x)dP(z, t, \omega),
\end{equation}
where 
\begin{equation*} 
h(z,t, a) = \text{sgn}\sin(z\|\omega\|_1t+b(\omega))\ (z a\cdot x-t)^2_{+}
\end{equation*}
and 
\begin{equation*}
dP(z,t,\omega) = \frac{1}{v}|\sin(z\|\omega\|_1t+b(\omega))|\|\omega\|^3_1|\calF(f)(\omega)|dt d\omega,
\end{equation*}
\begin{align*}
& v = \int_{\mathbb{R}^d}\int_{0}^1[|\sin(\|\omega\|_1t+b(\omega))|+ \\ & \qquad |\sin(\|\omega\|_1t-b(\omega))|]\|\omega\|^3_1|\calF(f)(\omega)|dtd\omega \leq 2v_{f,3}.
\end{align*}
The result follows from an application of \prettyref{thm:general}.
\end{proof}

\begin{remark} \label{rmk:sample}
By slightly modifying the definition of $ h $ from the proofs of \prettyref{thm:ramp} and \prettyref{thm:smooth_spline} (in particular, multiplying it by a sinusoidal function of $ \omega $ and $ t $), it suffices to sample instead from the density $ dP(t, \omega) = \frac{\|\omega\|^s_1|\calF(f)(\omega)|}{v_{f, s}}dt d\omega $ on $ [0, 1] \times \mathbb{R}^d $.
\end{remark}

\begin{remark}
For unit bounded $ x $, the expression $ e^{i \omega \cdot x} - i\omega \cdot x - 1 $ is bounded in magnitude by $ \|\omega\|^2_1 $, so one only needs Fourier representation of $ f(x) - x \cdot \nabla f(0) - f(0) $ when using the integrability with the $ \|\omega\|^2_1 $ factor. Similarly, $ e^{i \omega \cdot x} + (\omega \cdot x)^2/2 - i\omega \cdot x - 1 $  is bounded in magnitude by $ \|\omega\|^3 $, so one only needs Fourier representation of  $ f(x) - x^T \nabla \nabla^T f(0) x - x \cdot \nabla f(0) -1 $ when using the integrability with the $ \|\omega\|^3_1 $ factor.
\end{remark}

\begin{remark}
Note that in \prettyref{thm:ramp} and \prettyref{thm:smooth_spline}, we work with integrals with respect to the absolutely continuous measure $ d \calF(f)(\omega) $. In general, a (complex) Fourier measure $ d\calF(f)(\omega) $ does not need to be absolutely continuous.  For instance, it can be discrete on a lattice of values of $ \omega $, associated with a multivariate Fourier series representation for bounded domains  $ x $ (and periodic extensions thereof). Indeed, for bounded domains, one might have access to both Fourier series and Fourier transforms of extensions of $ f $ to $ \mathbb{R}^d $.  The best extension is one that gives the smallest Fourier norm $ \int_{\mathbb{R}^d} \|\omega\|^s_1 |d\calF(f)(\omega)| $. For further discussion along these lines, see \cite{Barron1993}.
\end{remark}

Next, we investigate the optimality of the rates from \prettyref{sec:uniform}.
\subsection{Lower bounds}

Let $ \mathcal{H}_s = \{ x\mapsto \eta( a\cdot x - t)^{s-1}_{+}: \| a\|_1 \leq 1,\  0 \leq t \leq 1, \  \eta\in\{-1,+1\} \} $ and for $ p \in [2, +\infty] $ let $ \mathcal{F}^s_p $ denote the closure of the convex hull of $ \mathcal{H}_s $ with respect to the $ \|\cdot\|_p $ norm on $ L^p(D,P) $ for $ p $ finite, where $ P $ is the uniform probability measure on $ D $, and $ \|\cdot\|_{\infty} $ (the supremum norm over $ D $) for $ p = +\infty $. We let $ \mathcal{C}^s_{m} $ denote the collection of all convex combinations of $ m $ terms from $ \mathcal{H}_s $. By \prettyref{thm:ramp} and \prettyref{thm:smooth_spline}, after possibly subtracting a linear or quadratic term, $ f/(2v_{f,2}) $ and $ f/v_{f,3} $ belongs to $ \mathcal{F}^2_p $ and $ \mathcal{F}^3_p $, respectively.
For $ p \in [2,+\infty] $ and $ \epsilon > 0 $, we define the $ \epsilon $-covering number $ N_p(\epsilon) $ by
\begin{equation*}
\min\{n: \exists\  \mathcal{F}\subset\mathcal{F}^s_p, \  |\mathcal{F}| = n,\  \text{s.t.} \inf_{f^{\prime}\in\mathcal{F}}\sup_{f\in\mathcal{F}^s_p}\|f-f^{\prime}\|_p \leq \epsilon \}.
\end{equation*}
and the $ \epsilon $-packing number $ M_p(\epsilon) $ by
\begin{equation*}
\max\{n: \exists\  \mathcal{F}\subset\mathcal{F}^s_p, \  |\mathcal{F}| = n,\  \text{s.t.} \inf_{f, f^{\prime}\in \mathcal{F}}\|f-f^{\prime}\|_p > \epsilon \}.
\end{equation*}

\prettyref{thm:general} implies that $ \inf_{f_m\in\mathcal{C}^s_m}\sup_{f\in\mathcal{F}^s_{\infty}}\|f-f_m\|_{\infty} $ achieves the bounds as stated therein.

\begin{theorem} \label{thm:lower}
For $ p \in [2,+\infty] $ and $ s \in \{ 2, 3\} $,
\begin{align*}
\inf_{f_m\in\mathcal{C}^s_m}\sup_{f\in\mathcal{F}^s_p}\|f-f_m\|_p \geq (Amd^{2s+1}\log(md))^{-1/2-s/d},
\end{align*}
for some universal positive constant $ A $.
\end{theorem}

Ignoring the dependence on $ d $ and logarithmic factors in $ m $, this result coupled with \prettyref{thm:general} implies that $ \inf_{f_m\in\mathcal{C}^2_m}\sup_{f\in\mathcal{F}^2_p}\|f-f_m\|_p $ is between $ m^{-1/2-2/d} $ and $ m^{-1/2-1/d} $; for large $ d $, the rates are essentially the same. Compare this with \cite[Theorem 4]{Makovoz1996} or \cite[Theorem 3]{Barron1992}, where a lower bound of $ c(\delta, d)\ m^{-1/2-1/d-\delta} $, $ \delta > 0 $ arbitrary, was obtained for approximants of the form \prettyref{eq:linear} with Lipschitz $ \phi $, but with inner parameter vectors of unbounded $ \ell^1 $ norm.

We only give the proof of \prettyref{thm:lower} for $ s = 2 $, since the other case $ s = 3 $ is handled similarly. First, we provide a few ancillary results that will be used later on. The next result is contained in \cite[Lemma 4.2]{Kurkova2007} and is useful for giving a lower bound on $ M_p(\epsilon) $.

\begin{lemma} \label{lmm:lower}
Let $ H $ be a Hilbert space equipped with a norm $ \|\cdot\| $ and containing a finite set $ \mathcal{H} $ with the following properties.
\begin{enumerate}[(i)]
\item $ |\mathcal{H}| \geq 3 $,
\item
$ \sum_{h,h^{\prime}\in \mathcal{H},\  h\neq h^{\prime}}|\langle h, h^{\prime} \rangle | \leq \delta^2 $
\item
$  \delta^2 \leq \min_{h\in \mathcal{H}}\|h\|^2 $
\end{enumerate}
Then there exists a collection $ \Omega \subset \{0,1 \}^{|\mathcal{H}|} $ with cardinality at least $ 2^{(1-H(1/4))|\mathcal{H}|-1} $, where $ H(1/4) $ is the entropy of a Bernoulli random variable with success probability $ 1/4 $, such that each pair of elements in the set $ \mathcal{F} = \left\{ \frac{1}{|\mathcal{H}|}\sum_{h\in \mathcal{H}}\omega_h h: (\omega_h : h\in \mathcal{H}) \in \Omega \right\} $ is separated by at least $  \frac{1}{2}\sqrt{\frac{\min_{h\in \mathcal{H}}\|h\|^2-\delta^2}{|\mathcal{H}|}} $ in $ \|\cdot\| $.
\end{lemma}

\begin{lemma} \label{lmm:sine}
If $ \theta $ belongs to $ [R]^d = \{1,2,\dots,R\}^d $, $ R \in \mathbb{Z}^{+} $, then the collection of functions 
\begin{equation*}
\mathcal{H} = \{ x \mapsto \sin(\pi\theta\cdot x)/(4\pi\|\theta\|^2_1): \theta \in [R]^d \} 
\end{equation*}
satisfies the assumption of \prettyref{lmm:lower} with $ H = L^2(D,P) $, where $ P $ is the uniform probability measure on $ D $. Moreover, $ |\mathcal{H}| = R^d $, $ \delta = 0 $, $ \min_{h\in \mathcal{H}}\|h\| = 1/(4\sqrt{2}\pi d^2R^2) $, and $ \mathcal{F} \subset \mathcal{F}^1_p $ for all $ p \in [2,+\infty] $.
Consequently, if $ \epsilon = 1/(8\sqrt{2}\pi d^2R^{2+d/2}) $, then 
\begin{align} 
\log M_p(\epsilon) & \geq  (\log 2)(1-H(1/4))\left(8\epsilon \sqrt{2}\pi d^2\right)^{-\frac{2d}{4+d}} - 1 \nonumber \\
& \geq \left(c\epsilon d^2\right)^{-\frac{2d}{4+d}}, \label{eq:packing-lower}
\end{align}
for some universal constant $ c > 0 $.
\end{lemma}
\begin{proof}
We first observe the identity
\begin{align*}
& \sin(\pi\theta\cdot x)/(4\pi\|\theta\|^2_1) = \theta\cdot x/(4\pi\|\theta\|^2_1) + \\ & \qquad \frac{\pi}{4}\int_{0}^{1}[(- a\cdot x-t)_{+}-( a\cdot x-t)_{+}]\sin(\pi\|\theta\|_1t)dt,
\end{align*}
where $  a =  a(\theta) = \theta/\|\theta\|_1 $. Note that above integral can also be written as an expectation of 
\begin{equation*}
-z \ \text{sgn}(\sin(\pi\|\theta\|_1t)) \ (z a\cdot x - t)_{+} \in \mathcal{H}_2
\end{equation*}
with respect to the density 
\begin{align*}
p_{\theta}(z,t) = \frac{\pi}{4}|\sin(\pi \|\theta\|_1t)|,
\end{align*}
on $ \{-1,1\}\times[0,1] $. The fact that $ p_{\theta} $ integrates to one is a consequence of the identity
\begin{equation*}
\int_0^1|\sin(\pi\|\theta\|_1t)|dt = 2/\pi.
\end{equation*}
Since $ \int_{D}|\sin(\pi\theta\cdot x)|^2dP(x) = 1/2 $, each member of $ \mathcal{H} $ has norm equal to $ 1/(4\sqrt{2}\pi\|\theta\|^2_1) $ and each pair of elements is orthogonal so that $ \delta = 0 $. Integrations over $ D $ involving $ \sin(\pi\theta\cdot x) $ are easiest to see using an instance of Euler's formula, viz., $ \sin(\alpha\cdot x) = \frac{1}{2 i}(\prod_{k=1}^de^{i\alpha(k) x(k) } - \prod_{k=1}^de^{-i\alpha(k)x(k)}) $.
\end{proof}

\begin{proof}[Proof of \prettyref{thm:lower}]
Let $ A > 0 $ be arbitrary. Suppose contrary to the hypothesis,
\begin{align*}
\inf_{f_m\in\mathcal{C}^2_m}\sup_{f\in\mathcal{F}^2_p}\|f-f_m\|_p & < (Amd^5\log(md))^{-1/2-2/d} \\ 
& \triangleq \epsilon_0/3.
\end{align*}

Note that each element of $ \mathcal{C}^2_m $ has the form $ \sum_{k=1}^m \lambda_k h_k $, where $ \sum_{k=1}^m \lambda_k = 1 $ and $ h_k \in \mathcal{H}_s $. Next, consider the subcollection $ \widetilde{\mathcal{C}}^2_m $ with elements of the form $ \sum_{k=1}^m \widetilde{\lambda}_k \widetilde{h}_k $, where $ \widetilde{\lambda}_k $ belongs to an $ \epsilon_0/3 $-net $ \widetilde{\mathcal{P}} $ of the $ m-1 $ dimensional probability simplex $ \mathcal{P}_m $ and $ \widetilde{h}_k $ belongs to an $ \epsilon_0/3 $-net $ \widetilde{\mathcal{H}} $ of $ \mathcal{H}_s $.
By a stars and bars argument, there are at most $ |\widetilde{\mathcal{P}}|\binom{m+|\mathcal{H}|-1}{m} $ such functions. Furthermore, since $ \sup_{h \in\mathcal{H}_s}\| h \|_{\infty} \leq 1$, we have
\begin{align*}
\inf_{f_m\in\widetilde{\mathcal{C}}^2_m}\sup_{f\in\mathcal{F}^2_p}\|f-f_m\|_2 & \leq \inf_{f_m\in\mathcal{C}^2_m}\sup_{f\in\mathcal{F}^2_p}\|f-f_m\|_2 + \\ & \qquad \inf_{\widetilde{h} \in \widetilde{\mathcal{H}}}\sup_{h\in\mathcal{H}_s}\|h-\widetilde{h} \|_2 + \\ & \qquad\qquad \inf_{\widetilde{\lambda} \in \widetilde{\mathcal{P}}}\sup_{\lambda \in \mathcal{P}_m}\|\lambda - \widetilde{\lambda} \|_1 \\
& < \epsilon_0/3 + \epsilon_0/3 + \epsilon_0/3 = \epsilon_0.
\end{align*}
Since $ |\widetilde{\mathcal{H}}| \asymp \epsilon^{-d-1}_0 $ and $ |\widetilde{\mathcal{P}}| \asymp \epsilon_0^{-m+1} $, it follows that
\begin{align}
\log N_p(\epsilon_0) & \leq \log |\widetilde{\mathcal{C}}^2_m| \nonumber \\
& \leq c_0\log \left[\epsilon_0^{-m-1}\binom{m+c_1\epsilon_0^{-d-1}-1}{m} \right] \nonumber \\
& \leq c_2dm\log(1/\epsilon_0) \nonumber \\
& \leq c_3dm\log(Adm), \label{eq:upperb}
\end{align}
for some positive universal constants $ c_0 > 0 $, $ c_1 > 0 $, $ c_2 > 0 $, and $ c_3 > 0 $.

On the other hand, using \prettyref{eq:packing-lower} from \prettyref{lmm:sine} coupled with the fact that $ N_p(\epsilon_0) \geq M_p(2\epsilon_0) $, we have
\begin{align}
\log N_p(\epsilon_0) & \geq \log M_p(2\epsilon_0) \nonumber \\ 
& \geq \left(2c\epsilon_0 d^2\right)^{-\frac{2d}{4+d}} \nonumber \\ 
& \geq c_4Adm\log(dm), \label{eq:lowerb}
\end{align}
for some universal constant $ c_4 > 0 $.
Combining \prettyref{eq:upperb} and \prettyref{eq:lowerb}, we find that
\begin{equation*}
c_4 Adm\log(dm) \leq c_3dm\log(Adm).
\end{equation*}
If $ A $ is large enough (independent of $ m $ or $ d $), we reach a contradiction. This proves the lower bound.
\end{proof}

\section{$ L^2 $ approximation with bounded $ \ell^0 $ and $ \ell^1 $ norm} \label{sec:squared}
In \prettyref{sec:uniform}, we explored conditions for which good approximation in $ L^{\infty}(D) $ could be achieved even with $ \ell^1 $ controls on the inner parameter vectors. In this section, we show how similar statements can be made in $ L^2(D) $, but with control on the $ \ell^0 $ norm as well. Note that unlike \prettyref{thm:general}, we see in \prettyref{thm:sparse} how the smoothness of the activation function directly affects the rate of approximation. The proof is obtained by applying the Jones-Barron probabilistic method in two stages (similar to two-stage cluster sampling), first on the outer layer coefficients, and then on the inner layer coefficients.

\begin{theorem} \label{thm:sparse}
Suppose $ f $ admits an integral representation
\begin{equation*}
f(x) = v\int_{[0,1]\times \{ a:\| a\|_1 = 1\}}\eta(t, a) \ ( a\cdot x - t)^{s-1}_{+}dP(t, a),
\end{equation*}
for $ x $ in $ D = [-1,1]^d $ and $ s \in \{ 2, 3 \} $,
where $ P $ is a probability measure on $ [0,1]\times \{ a \in \mathbb{R}^d:\| a\|_1 = 1\} $ and $ \eta(t, a) $ is either $ -1 $ or $ +1 $. 
There exists a linear combination of ridge functions of the form
\begin{equation*}
f_{m, m_0}(x) = \frac{v}{m}\sum_{k=1}^{m} b_k\left(a_k \cdot x-t_k\right)^{s-1}_{+},
\end{equation*}
where $ \|a_k\|_0 \leq m_0 $, $ \|a_k\|_1 = 1 $, and $  b_k \in \{-1, +1\} $ such that
\begin{equation*}
\| f -  f_{m, m_0} \|_2 \leq v\sqrt{\frac{1}{m} + \frac{1}{m^{s-1}_0}}.
\end{equation*}
Furthermore, the same rates for $ s = 2 $ or $ s = 3 $ are achieved for general $ f $ adjusted by a linear or quadratic term with $ v = 2v_{f, 2} < +\infty $ or $ v = v_{f, 3} < +\infty $, respectively.
\end{theorem}

\begin{remark}
In particular, taking $ m_0 = \sqrt{m} $, it follows that there exists an $ m $-term linear combination of squared ReLU ridge functions, with $ \sqrt{m} $-sparse inner parameter vectors, that approximates $ f $ with $ L^2(D) $ error at most $ \sqrt{2}vm^{-1/2} $. In other words, the $ L^2(D) $ approximation error is inversely proportional to the inner layer sparsity and it need only be sublinear in the outer layer sparsity.
\end{remark}

\begin{proof}
Take a random sample $  \underline{a} = \{ (t_k,  a_k)^{\prime} \}_{1 \leq k\leq m} $ from $ P $. Given $  \underline{a} $, take a random sample $ \underline{\widetilde{a}} = \{ \widetilde{a}_{\ell, k} \}_{1 \leq \ell \leq m_0, \  1 \leq k \leq m} $, where $ \mathbb{P}[ \widetilde{a}_{\ell, k} = \text{sgn}( a_k(j))e_j ] = |a_k(j)| $ for $ j = 1, \dots, d $, $ a_k = (a_k(1), \dots, a_k(d))^{\prime} $, and $ e_j $ is the $j$-th standard basis vector for $ \mathbb{R}^d $. Note that 
\begin{equation} \label{eq:mean}
\mathbb{E}_{\underline{\widetilde{a}}\mid  \underline{a}}[\widetilde{a}_{\ell, k}] = a_k
\end{equation}
and 
\begin{align} 
\Var_{\underline{\widetilde{a}}\mid  \underline{a}}[\widetilde{a}_{\ell, k}\cdot x] & \leq \mathbb{E}_{\underline{\widetilde{a}}\mid  \underline{a}}[\widetilde{a}_{\ell, k}\cdot x]^2 = \sum_{j=1}^d |a_k(j)||x(j)|^2 \nonumber \\ & \leq \|a_k\|_1 \|x\|^2_{\infty} \leq 1. \label{eq:var}
\end{align}
Define
\begin{equation} \label{eq:linear2}
 \overline{f}_{m, m_0}(x) = \frac{v}{m}\sum_{k=1}^{m}\eta(t_k,  a_k) \left(\frac{1}{m_0}\sum_{\ell = 1}^{m_0}\widetilde{a}_{\ell, k} \cdot x-t_k\right)^{s-1}_{+}.
\end{equation}
By the bias-variance decomposition,
\begin{equation*}
\mathbb{E}\| f -  \overline{f}_{m, m_0}\|^2_2 = \mathbb{E} \|  \overline{f}_{m, m_0} -  \mathbb{E}\overline{f}_{m, m_0}\|^2_2 + \| f -  \mathbb{E}\overline{f}_{m, m_0}\|^2_2.
\end{equation*}
Note that $ \mathbb{E} \|  \overline{f}_{m, m_0} -  \mathbb{E}\overline{f}_{m, m_0}\|^2_2 \leq \frac{v^2}{m} $. Next, observe that
\begin{align*}
& f(x) - \mathbb{E}\overline{f}_{m, m_0}(x)
 = \frac{v}{m}\sum_{k=1}^{m}\mathbb{E}_{ \underline{a}}\Bigg[\eta(t_k,  a_k) \times \\ & \mathbb{E}_{\underline{\widetilde{a}}\mid  \underline{a}}\left( \left( a_k \cdot x-t_k\right)^{s-1}_{+} - \left(\frac{1}{m_0}\sum_{\ell = 1}^{m_0}\widetilde{a}_{\ell, k} \cdot x-t_k\right)^{s-1}_{+}\right)\Bigg],
\end{align*}
which, by an application of the triangle inequality, implies that
\begin{align*}
& |f(x) - \mathbb{E}\overline{f}_{m, m_0}(x)| \leq \frac{v}{m}\sum_{k=1}^{m} \\ & 
\mathbb{E}_{ \underline{a}}\left|\left( a_k \cdot x-t_k\right)^{s-1}_{+} - \mathbb{E}_{\underline{\widetilde{a}}\mid  \underline{a}}\left(\frac{1}{m_0}\sum_{\ell = 1}^{m_0}\widetilde{a}_{\ell, k} \cdot x-t_k\right)^{s-1}_{+}\right|.
\end{align*}
Next, we use the following two properties of $ (z)^{s-1}_{+} $: for all $ z $ and $ z' $ in $ \mathbb{R} $,
\begin{align}
\label{eq:T1} |(z)_+ - (z')_+| & \leq |z-z'|, \\
\label{eq:T2} |(z)^2_+ - (z')^2_+ - 2(z-z')(z')_+ | & \leq |z-z'|^2.
\end{align}
If $ s = 2 $, we have by \prettyref{eq:T1}, \prettyref{eq:mean}, and \prettyref{eq:var} that
\begin{align*}
& \mathbb{E}_{ \underline{a}}\left|\left( a_k \cdot x-t_k\right)_{+} - \mathbb{E}_{\underline{\widetilde{a}}\mid  \underline{a}}\left(\frac{1}{m_0}\sum_{\ell = 1}^{m_0}\widetilde{a}_{\ell, k} \cdot x-t_k\right)_{+}\right| \leq \\ & \mathbb{E}_{ \underline{a}}\mathbb{E}_{\underline{\widetilde{a}}\mid  \underline{a}}\left| a_k \cdot x - \frac{1}{m_0}\sum_{\ell = 1}^{m_0}\widetilde{a}_{\ell, k} \cdot x \right| \leq \\ & \mathbb{E}_{ \underline{a}}\sqrt{\mathbb{E}_{\underline{\widetilde{a}}\mid  \underline{a}}\left| a_k \cdot x - \frac{1}{m_0}\sum_{\ell = 1}^{m_0}\widetilde{a}_{\ell, k} \cdot x \right|^2} = \\ &
 \mathbb{E}_{ \underline{a}}\sqrt{\frac{\Var_{\underline{\widetilde{a}}\mid  \underline{a}}[\widetilde{a}_{\ell, k}\cdot x]}{m_0}}  \leq \frac{1}{\sqrt{m_0}}.
\end{align*}
This shows that $ \| f -  \mathbb{E}\overline{f}_{m, m_0}\|^2_2 \leq \frac{v^2}{m_0} $.
If $ s = 3 $, we have from \prettyref{eq:T2}, \prettyref{eq:mean}, and \prettyref{eq:var} that
\begin{align*}
& \mathbb{E}_{ \underline{a}}\left|\left( a_k \cdot x-t_k\right)^2_{+} - \mathbb{E}_{\underline{\widetilde{a}}\mid  \underline{a}}\left(\frac{1}{m_0}\sum_{\ell = 1}^{m_0}\widetilde{a}_{\ell, k} \cdot x-t_k\right)^2_{+}\right| \leq \\ & \mathbb{E}_{ \underline{a}}\mathbb{E}_{\underline{\widetilde{a}}\mid  \underline{a}}\left| a_k \cdot x - \frac{1}{m_0}\sum_{\ell = 1}^{m_0}{\widetilde{a}}_{\ell, k} \cdot x \right|^2 = \\ &
 \mathbb{E}_{ \underline{a}}\left[\frac{\Var_{\underline{\widetilde{a}}\mid  \underline{a}}[\widetilde{a}_{\ell, k}\cdot x]}{m_0}\right]  \leq \frac{1}{m_0}.
\end{align*}
This shows that $ \| f -  \mathbb{E}\overline{f}_{m, m_0}\|^2_2 \leq \frac{v^2}{m^2_0} $. Since these bounds hold on average, there exists a realization of \prettyref{eq:linear2} for which the bounds are also valid. Note that the vector $ \frac{1}{m_0}\sum_{\ell = 1}^{m_0}\widetilde{a}_{\ell, k} $ has $ \ell^0 $ norm at most $ m_0 $ and unit $ \ell^1 $ norm.

The fact that the bounds also hold for $ f $ adjusted by a linear or quadratic term (under an assumption of finite $ v_{f, 2} $ or $ v_{f, 3} $) follows from \prettyref{eq:ramp_rep} and \prettyref{eq:smooth_spline_rep}.
\end{proof}

\section*{Acknowledgment}

The authors would like to thank the anonymous reviewers whose detailed feedback led to dramatic improvements to this paper. They also thank Joowon Kim for helpful comments on earlier drafts of this manuscript. 

\bibliographystyle{plain}
\bibliography{Uniform2016References_Final}

\end{document}